\def\isarxiv{1} 

\ifdefined\isarxiv
\documentclass[11pt]{article}
\usepackage[numbers]{natbib}

\else
\documentclass[sigconf,anonymous,review]{acmart}
\fi

\usepackage{amsmath}
\usepackage{amsthm}
\ifdefined\isarxiv 
\usepackage{amssymb}
\else
\fi
\usepackage{algorithm}
\usepackage{subfig}
\usepackage{algpseudocode}
\usepackage{graphicx}
\usepackage{grffile}
\usepackage{wrapfig,epsfig}
\usepackage{url}
\usepackage{xcolor}
\usepackage{epstopdf}

\usepackage{bbm}
\usepackage{dsfont}
 
\allowdisplaybreaks

\ifdefined\isarxiv

\usepackage{tikz}
\usepackage{hyperref}  
\hypersetup{colorlinks=true,citecolor=blue,linkcolor=blue} 
\usetikzlibrary{arrows}
\usepackage[margin=1in]{geometry}

\else

\usepackage{microtype}
\usepackage{hyperref}
\definecolor{mydarkblue}{rgb}{0,0.08,0.45}
\hypersetup{colorlinks=true, citecolor=mydarkblue,linkcolor=mydarkblue}

\fi
 
\graphicspath{{./figs/}}

\theoremstyle{plain}
\newtheorem{theorem}{Theorem}[section]
\newtheorem{lemma}[theorem]{Lemma}
\newtheorem{definition}[theorem]{Definition}

\newcommand{\N}{\mathcal{N}}
\newcommand{\R}{\mathbb{R}}

\renewcommand{\d}{\mathrm{d}}

\DeclareMathOperator*{\E}{{\mathbb{E}}}

\makeatletter
\newcommand*{\RN}[1]{\expandafter\@slowromancap\romannumeral #1@}
\makeatother

\usepackage{lineno}

\begin{document}

\ifdefined\isarxiv

\date{}

\title{Force Matching with Relativistic Constraints: A Physics-Inspired Approach to Stable and Efficient Generative Modeling}


\author{
Yang Cao\thanks{\texttt{ ycao4@wyomingseminary.org}. Wyoming Seminary}
\and
Bo Chen\thanks{\texttt{ bc7b@mtmail.mtsu.edu}. Middle Tennessee State University.}
\and
Xiaoyu Li\thanks{\texttt{
xli216@stevens.edu}. Stevens Institute of Technology.}
\and
Yingyu Liang\thanks{\texttt{
yingyul@hku.hk}. The University of Hong Kong. \texttt{
yliang@cs.wisc.edu}. University of Wisconsin-Madison.} 
\and
Zhizhou Sha\thanks{\texttt{
shazz20@mails.tsinghua.edu.cn}. Tsinghua University.}
\and
Zhenmei Shi\thanks{\texttt{
zhmeishi@cs.wisc.edu}. University of Wisconsin-Madison.}
\and
Zhao Song\thanks{\texttt{ magic.linuxkde@gmail.com}. Simons Institute for the Theory of Computing, University of California, Berkeley.}
\and
Mingda Wan\thanks{\texttt{
dylan.r.mathison@gmail.com}. Anhui University.}
}

\else

\title{Force Matching with Relativistic Constraints: A Physics-Inspired Approach to Stable and Efficient Generative Modeling}

\author{%
}

\fi

\ifdefined\isarxiv
\begin{titlepage}
  \maketitle
  \begin{abstract}
This paper introduces Force Matching (ForM), a novel framework for generative modeling that represents an initial exploration into leveraging special relativistic mechanics to enhance the stability of the sampling process. By incorporating the Lorentz factor, ForM imposes a velocity constraint, ensuring that sample velocities remain bounded within a constant limit. This constraint serves as a fundamental mechanism for stabilizing the generative dynamics, leading to a more robust and controlled sampling process. 
We provide a rigorous theoretical analysis demonstrating that the velocity constraint is preserved throughout the sampling procedure within the ForM framework. To validate the effectiveness of our approach, we conduct extensive empirical evaluations. On the \textit{half-moons} dataset, ForM significantly outperforms baseline methods, achieving the lowest Euclidean distance loss of \textbf{0.714}, in contrast to vanilla first-order flow matching (5.853) and first- and second-order flow matching (5.793). Additionally, we perform an ablation study to further investigate the impact of our velocity constraint, reaffirming the superiority of ForM in stabilizing the generative process.
The theoretical guarantees and empirical results underscore the potential of integrating special relativity principles into generative modeling. Our findings suggest that ForM provides a promising pathway toward achieving stable, efficient, and flexible generative processes. This work lays the foundation for future advancements in high-dimensional generative modeling, opening new avenues for the application of physical principles in machine learning.

  \end{abstract}
  \thispagestyle{empty}
\end{titlepage}

{\hypersetup{linkcolor=black}
\tableofcontents
}
\newpage

\else

\begin{abstract}
This paper introduces Force Matching (ForM), a novel framework for generative modeling that represents an initial exploration into leveraging special relativistic mechanics to enhance the stability of the sampling process. By incorporating the Lorentz factor, ForM imposes a velocity constraint, ensuring that sample velocities remain bounded within a constant limit. This constraint serves as a fundamental mechanism for stabilizing the generative dynamics, leading to a more robust and controlled sampling process. 
We provide a rigorous theoretical analysis demonstrating that the velocity constraint is preserved throughout the sampling procedure within the ForM framework. To validate the effectiveness of our approach, we conduct extensive empirical evaluations. On the \textit{half-moons} dataset, ForM significantly outperforms baseline methods, achieving the lowest Euclidean distance loss of \textbf{0.714}, in contrast to vanilla first-order flow matching (5.853) and first- and second-order flow matching (5.793). Additionally, we perform an ablation study to further investigate the impact of our velocity constraint, reaffirming the superiority of ForM in stabilizing the generative process.
The theoretical guarantees and empirical results underscore the potential of integrating special relativity principles into generative modeling. Our findings suggest that ForM provides a promising pathway toward achieving stable, efficient, and flexible generative processes. This work lays the foundation for future advancements in high-dimensional generative modeling, opening new avenues for the application of physical principles in machine learning.

\end{abstract}




\maketitle 
\fi

\section{Introduction}
The field of generative modeling has witnessed significant progress with the advent of sophisticated techniques that leverage neural networks to synthesize high-quality data. Recent methods such as Diffusion Models (DM) \cite{swm+15,hja20,dn21,rbl+22,zlcz23,zlke23}, Flow Matching (FM) \cite{lcb+22,lgl22,ekb+24}, and the more recent Equilibrium Dynamics Model (EDM) \cite{kaal22, kal+24} have emerged as prominent approaches, each exploring distinct generative paradigms. These techniques differ fundamentally in how they utilize neural networks to evolve data representations over time: while Diffusion Models rely on iterative transformations of a Gaussian-initialized distribution, EDM employs an Ordinary Differential Equation (ODE) for continuous evolution, and FM directly predicts the data's velocity using a neural-network-based velocity field. Such diversity in generative approaches has motivated the need for a unified perspective that can bridge these conceptual differences.

In pursuit of this unification, TrigFlow \cite{ls24} was proposed as a generalized framework that provides a continuous generative process capable of transitioning between the behaviors of EDM and FM. By leveraging trigonometric parameterization, TrigFlow formulates data generation through a combination of trigonometric components, allowing for a flexible representation that captures the strengths of both paradigms. The framework introduces a trigonometric-based parameterization of the generative process, a loss function aligning with diffusion methods, and a probability flow ODE, thereby offering a more comprehensive understanding of generative modeling techniques and providing a foundation for further advancements in this domain.

Building on this unification perspective, this paper introduces Force Matching (ForM) as a novel generative framework inspired by principles of relativistic mechanics to stabilize the sampling process. By incorporating relativistic constraints through the Lorentz factor, ForM ensures stable sampling dynamics, limiting the velocity of generated samples to avoid instability. 
We establish that ForM is well-aligned with consistency models, suggesting its potential to enhance scalable generative modeling solutions.
Our contributions can be summarized as follows:
\begin{itemize}
    \item We propose \textbf{Force Matching (ForM)}, a novel generative modeling framework inspired by relativistic mechanics, which ensures stable sampling by constraining sample velocities through the Lorentz factor.
    \item We establish theoretical foundations for ForM, 
    highlighting its flexibility and scalability.
    \item We conduct extensive empirical evaluations, showing that ForM outperforms baseline flow matching methods in generative tasks and validating the effectiveness of its velocity constraint through ablation studies.
\end{itemize}

These contributions illustrate the promise of force-based methods in generative modeling, emphasizing their capability for stable, efficient, and flexible sampling. This work not only extends our understanding of generative techniques but also lays the foundation for exploring novel high-dimensional generative frameworks that effectively integrate stability and efficiency.

{\bf Roadmap.} In Section~\ref{sec:related_work}, we introduce related work of generative models and flow matching. Then, Section~\ref{sec:preli} introduces the preliminary of Force Matching. We then propose the Force Matching architecture in Section~\ref{sec:form}. Section~\ref{sec:exp} demonstrates empirical experiments of Force Matching, and Section~\ref{sec:ablation} performs ablation study of Force Matching. Finally, we conclude this paper in Section~\ref{sec:conclusion}.
\section{Related work} \label{sec:related_work}

\paragraph{Generating Models.}
Generative models have made significant progress over the last decade, enabling diverse applications such as image synthesis, text generation, and data augmentation. One of the foundational models in this area is the Generative Adversarial Network (GAN) introduced by Goodfellow et al. (2014), which consists of a generator and a discriminator that compete in a zero-sum game, thereby leading to the generation of realistic data samples \cite{goodfellow2014generative}. GANs have inspired a variety of derivative architectures aimed at improving stability and quality, including Wasserstein GANs (WGAN), which address the instability issues of GANs by employing a different distance metric \cite{arjovsky2017wasserstein}. Conditional GANs (cGANs) extend the GAN framework to generate data conditioned on additional information, making them more controllable \cite{mirza2014conditional}. 

Variational Autoencoders (VAEs), proposed by Kingma and Welling (2013), offer another generative approach that combines variational inference with neural networks to learn a latent variable model of data \cite{kingma2013auto}. Unlike GANs, VAEs maximize a lower bound on the log-likelihood of the data, allowing for a more principled probabilistic interpretation. The introduction of the reparameterization trick was key to making VAEs feasible to train with stochastic gradient descent, which has had a considerable impact on the field of deep generative models. Recently, autoregressive models such as PixelCNN \cite{oord2016pixel} and Transformer-based models \cite{vaswani2017attention} have demonstrated impressive performance in tasks like image and text generation. These models learn to predict the next element in a sequence, thereby allowing them to generate samples one step at a time, which has proven particularly effective for generating sequential data such as text and audio. 

The development of large-scale language models like GPT-3 \cite{brown2020language} has further showcased the power of autoregressive architectures in generating coherent and contextually relevant long text, significantly advancing the state-of-the-art in natural language processing. Another line of work explores diffusion models, such as Denoising Diffusion Probabilistic Models (DDPMs), which have gained attention for their ability to generate high-quality images by modeling the process of gradually adding noise to data and then learning to reverse this process \cite{ho2020denoising}. These models provide an alternative to GANs by optimizing likelihood-based objectives, which makes training more stable. DDPMs have set new benchmarks for image generation quality, rivaling the output of GANs while avoiding some of their training difficulties. These developments collectively showcase the evolution of generative models from adversarial training with GANs to likelihood-based training with VAEs, autoregressive models, e.g, Visual autoregressive modeling (VAR) \cite{tjy+24}, and diffusion-based approaches, e.g., DDPM \cite{ho2020denoising}. Each of these methods contributes unique strengths and capabilities, advancing the scope and quality of generated data.

\paragraph{Flow Matching.}
Flow matching \cite{dlt+24,ylp+24,gdb+24,cgl+25_homo} is a key concept in fields such as optimal transport, computer vision, and machine learning, where it has been extensively studied and utilized to align two distributions effectively. The method has roots in the classic work on optimal transport theory, where Monge and Kantorovich initially laid out the foundational ideas for mapping mass distributions with minimal cost \cite{monge1781memoire, kantorovich1942transfer}. Building upon these ideas, Villani expanded the theoretical framework of optimal transport, leading to a rigorous mathematical foundation for flow matching and its related applications \cite{villani2008optimal}. 

Recent advances in machine learning have leveraged flow matching for deep generative modeling tasks. Denoising diffusion probabilistic models (DDPMs), for example, have drawn inspiration from flow-based methods to improve the stability and efficiency of training \cite{ho2020denoising}. Similarly, score-based generative models utilize a stochastic differential equation approach to approximate flows, effectively creating a flow-matching procedure for generating realistic data samples \cite{song2021score}. This approach has demonstrated considerable success in capturing complex, high-dimensional data distributions. Another relevant development in this domain is the introduction of continuous normalizing flows (CNFs) by Chen et al., which formulated generative modeling as a continuous-time flow process, further refining flow-matching techniques for density estimation and improving scalability \cite{chen2018neural}. Grathwohl et al. expanded upon this idea by demonstrating how flow matching could be combined with probability density estimation to achieve more efficient generative models \cite{grathwohl2018ffjord}. These works have collectively highlighted the flexibility of flow matching as a tool for a wide range of machine learning tasks, including unsupervised learning, density estimation, and data synthesis. Moreover, applications in computer vision often rely on flow matching to solve challenging problems such as image registration and optical flow estimation. For instance, deep learning-based approaches have integrated flow matching concepts to align images effectively, demonstrating significant improvements over traditional techniques \cite{dosovitskiy2015flownet, ilg2017flownet2}. FlowNet and its successor FlowNet2 provide compelling evidence of how flow matching can be operationalized within deep neural architectures to solve real-world vision tasks with state-of-the-art accuracy. Video Latent Flow Matching \cite{csy25,jsl+24,dsf23} incorporates flow matching for temporally coherent video generation.
Moreover, numerous recent works \cite{zcwt23,lss+24_relu,cls+24,lss+24_multi_layer,wxz+24,wcz+23,cgl+25_homo,xlc+24,wcy+23,sph+23,cxj24,fjl+24,kll+25,kll+25_tc,cll+25_var,kls+25_dpbloom,lsss24_dpntk,cll+25_deskrej,lssz24_gm,llss24_softmax,lzw+24,hwsl24,hwl+24,ssz+25_dit,ssz+25_prune} have significantly inspired and influenced our work.

\section{Preliminary} \label{sec:preli}

In Section~\ref{sub:notation}, we introduce all the notations we used in our paper. Then, in Section~\ref{sub:flow_matching}, we show the basic facts about flow matching. In Section~\ref{sub:special_relativity}, we present the basic background of special relativity and define the relativistic force.

\subsection{Notations} \label{sub:notation}

For any positive integer $n$, we use $[n]$ to denote set $\{1,2,\cdots, n\}$. 
For two vectors $x \in \R^n$ and $y \in \R^n$, we use $\langle x, y \rangle$ to denote the inner product between $x,y$.
For a vector $v \in \R^n$, we use $\|v\|_2$ to denote the $\ell_2$-norm of $v$.
We use ${\bf 1}_n$ to denote a length-$n$ vector where all the entries are ones.
We use the symbol $ \perp $ to represent a component that is perpendicular to the direction of velocity, as exemplified by $ a_{\perp t} $, which denotes the perpendicular acceleration. Similarly, the symbol $ \parallel $ is employed to indicate a component parallel to the direction of velocity, such as $ f_{\parallel t} $, which represents the parallel force. We use $\dot{x}_t$ to denote $\frac{\d x_t}{\d t}$, and $\ddot{x}_t$ to denote $\frac{\d^2 x_t}{\d t^2}$.

\subsection{Flow Matching} \label{sub:flow_matching}

Flow Matching (FM) \cite{lcb+22,lgl22} is a generative modeling technique that constructs a smooth, invertible (i.e., diffeomorphic) mapping from a simple prior distribution to a complex target distribution. In FM, a time-dependent mapping $Z_t$ is defined to evolve according to an ordinary differential equation (ODE) driven by a vector field:
\begin{align*}
    \frac{\d x_t}{\d t} = V_t(x_t), \quad t \in [0, T].
\end{align*}
The goal is to ensure that, at the terminal time $T$, the ODE transforms a sample $x_0$ from a simple distribution (e.g., a Gaussian) into a sample $x_T$ from the target data distribution $\mathcal{D}$.

To achieve this, Flow Matching (FM) constructs a stochastic interpolation between a sample $x_1 \sim \mathcal{D}$ and a sample $x_0$ drawn from a known prior distribution, typically $\N(0,I)$. The interpolation is defined as
\begin{align*}
    x_t := \alpha_t x_1 + \sigma_t x_0, \quad t\in [0,T],
\end{align*}
where the time-dependent coefficients $\alpha_t$ and $\sigma_t$ are chosen so that
\begin{align*}
    \alpha_0 = 0,\quad \sigma_0 = 1,\quad \alpha_T = 1,\quad \sigma_T = 0.
\end{align*}
Thus, at $t=0$ the interpolated sample is purely the prior ($x_0$), and at $t=T$ it becomes a data sample ($x_1$).

The instantaneous change of $x$ is obtained by differentiating the interpolation:
\begin{align*}
    \frac{\d x_t}{\d t} = \frac{\d \alpha_t}{\d t} x_1 + \frac{\d \sigma_t}{\d t} x_0.
\end{align*}

The vector field is approximated by a neural network $V_t(x_t)$ with learnable parameters $\theta$. The FM training objective is then given by
\begin{align*}
    \mathcal{L}_\mathrm{FM}(\theta) := \E_{t\sim {\sf Uniform}[0,T], x_1 \sim \mathcal{D}} [\| V_t(x_t) - v_t(x_t) \|_2^2 ].
\end{align*}
This loss ensures that the learned velocity field $V_t(x_t)$ closely tracks the conditional dynamics $v_t(x_t)$ along the interpolation path.

After training, samples are generated by solving the ODE
\begin{align*}
    \frac{\d x_t}{\d t} = V_t(x_t),
\end{align*}
starting from an initial sample $x_0 \sim \N(0,I)$. Integrating this ODE from $t=0$ to $t=T$ yields a sample $x_T$ that approximates a draw from the target distribution. This ODE-based formulation offers a flexible and powerful framework for modeling complex data distributions while naturally incorporating conditional sampling.

\subsection{Background on Special Relativity} \label{sub:special_relativity}

We first introduce several essential ideas of special relativity \cite{e+05}.

\begin{definition}[Lorentz Factor]
\label{def:LorentzFactor}
According to special relativity~\cite{e+05}, the Lorentz factor at lab time $t$ is given by
\begin{align*}
\gamma_t := \frac{1}{\sqrt{1 - {\|v_t^{\rm lab}\|_2^2}/{c^2}}},
\end{align*}
where $v_t^{\rm lab}$ is the velocity at lab frame of reference, $c = 3 \times 10^8$ is the speed of light in vacuum.
\end{definition}

Then, we introduce the proper time of special relativity.

\begin{definition}[Proper Time]
\label{def:ProperTime}
The proper time is defined as the time interval measured in the rest frame of a moving object according to special relativity~\cite{e+05}. The differential form of the proper time is given by
\begin{align*}
    \d \tau = \frac{\d t}{\gamma_t},
\end{align*}
where $\d t$ is the time interval in the laboratory frame of reference, and $\gamma_t$ is the Lorentz factor at time lab time $t$ as defined in Definition~\ref{def:LorentzFactor}.
\end{definition}

Next, we define the force under special relativity here.

\begin{definition}[Relativistic Force]
\label{def:RelativisticForce}
In the framework of special relativity, the \emph{local force} (i.e., the force measured in the instantaneous rest frame of the particle) denoted as $f^{\rm local}$ has
\begin{align}
    f^{\rm local} := \frac{\d p^{\rm lab}}{\d \tau}, \label{eq:f_local}
\end{align}
where $p^{\rm lab}$ is the momentum at lab frame of reference, $\tau$ denotes the proper time defined in Definition~\ref{def:ProperTime}.

The momentum in the lab frame is defined as
\begin{align}
    p^{\rm lab} := m^{\rm lab} v_t^{\rm lab}, \label{eq:p}
\end{align}
where $m^{\rm lab}$ is the mass at lab frame of reference, and $v_t^{\rm lab}$ is the velocity at lab frame of reference.
\end{definition}

We state an equivalence lemma. Due to the space limitation, we delayed the proofs into the appendix.
\begin{lemma}[Equivalent Form of Relativistic Force, informal version of Lemma~\ref{lem:equiv_relativistic_force:formal}]\label{lem:equiv_relativistic_force:informal}
Let $p^{\rm lab}$ be the momentum defined in Eq.~\eqref{eq:p}, $\gamma_t$ be the Lorentz factor at lab time $t$ defined in Definition~\ref{def:LorentzFactor}, $\tau$ denotes the proper time, $v_t^{\rm lab} = \dot{x}_t$ denotes the velocity, 
$a_t^{\rm lab} = \ddot{x}_t$ denotes the acceleration.
The relativistic force, defined as the time derivative of the momentum in the lab frame, can be written as
\begin{align*}
f^{\rm local} =  m^{\rm lab}  (\gamma_t a_t^{\rm lab} + \gamma_t^3 \frac{ \langle v_t^{\rm lab}, a_t^{\rm lab} \rangle}{c^2} v_t^{\rm lab}).
\end{align*}

\end{lemma} 
\section{Force Matching} \label{sec:form}

In this section, we introduce Force Matching (ForM), a new architecture for generative models, and provide its theoretical analysis. In Section~\ref{sub:obj}, we introduce the training objective of Force Matching. Then, in Section~\ref{sub:samp_ode}. In Section~\ref{sub:speed_limit}, we illustrate and discuss the speed limitation of ForM. In Section~\ref{sub:form_trig}, We show the interpolation path for ForM.

\subsection{Definition of Force Matching Objective} \label{sub:obj}

Next, we define the training objective of Force Matching.

\begin{definition}[Force Matching Objective]
\label{def:FormObjective}
The training objective of Force Matching (ForM) is defined by
\begin{align*}
     \mathcal{L}_{\rm ForM}(\theta) := \E_{t \sim {\sf Uniform}[0,T], x_1 \sim \mathcal{D}} 
    [\| F_t(x_t) - f_t(x_t)\|_2^2],   
\end{align*}
where $\mathcal{D}$ is the target data distribution, $f_t(x_t)$ is the target relativistic force defined in Definition~\ref{def:RelativisticForce}, and $F_t(x)$ is a trainable neural network parameterized with $\theta$.
\end{definition}

\subsection{Our Result I: Sampling ODE} \label{sub:samp_ode}

We define an ordinary differential equation (ODE) in order to get the position based on a given relativistic force. 

\begin{theorem}[Sampling ODE, informal version of Theorem~\ref{thm:ode_form:formal}]\label{thm:ode_form:informal}
    Giving the force at position $x_t$ denoted as $f_t(x_t)$, we could solve for ForM sampling path $x_t$ by the following ODE
    \begin{align*}
    \ddot{x}_t = \frac{1}{m^{\rm lab} \gamma_t}(f_t^{\rm local} - \frac{\langle v_t^{\rm lab}, f_t^{\rm local} \rangle}{c^2} v_t^{\rm lab}),
    \end{align*}
    where $x_0 \sim \N(0,I)$, $\dot{x}_0 = 0$.
\end{theorem}

Theorem~\ref{thm:ode_form:informal} shows how to derive the position $x_t$ from the relativistic force field $f_t(x_t)$. Unlike first-order flow-based methods, ForM naturally involves a second-order ODE because it encodes the evolution of both position and velocity under relativistic constraints. This allows for more expressive and physically-motivated sampling trajectories, where velocity constraints can help stabilize the generative process. Once a neural network $F_t(x)$ is trained to approximate $f_t(x)$, the sampling procedure integrates this second-order ODE to produce samples consistent with the target distribution.

\subsection{Our Result II: Speed Limit} \label{sub:speed_limit}

One property of relativistic mechanics is the velocity will always be under the constant $c$, which is the speed of light.

In reality, the speed of light is $c \approx 3 \times 10^8$. For any $v_t$, the speed $ \|v_t\|_2$ can approach but never exceed $c$. 
This property stabilizes the generating process. We formalize and prove this in Theorem~\ref{thm:vel:informal}.

\begin{theorem}[Speed Limit, informal version of Theorem~\ref{thm:vel:formal}] \label{thm:vel:informal}
For a ForM model with sampling path $x : [0,T) \to \R^n$, the velocity satisfies
\begin{align*}
\| \dot{x}_t \|_2 < c \quad, \forall t \in [0,T).
\end{align*}
\end{theorem}

It shows that the sample velocity remains strictly below $c$ at all times under relativistic constraints. Practically, this upper bound on velocity helps mitigate risks of numerical instability or ``exploding gradients" that can sometimes arise in diffusion- or flow-based generative models. By capping the speed of samples, ForM maintains a controlled and stable evolution in high-dimensional spaces. This provides a theoretical guarantee of safety against runaway behaviors, making the sampling process more robust.

\subsection{Our Result III: ForM with TrigFlow} \label{sub:form_trig}

The interpolation path of ForM is given by the following theorem.

\begin{theorem}[ForM with TrigFlow, informal version of Theorem~\ref{thm:form_trig:formal}] \label{thm:form_trig:informal}
We let $m = 1$ for simplicity in ForM. Giving a the interpolation $x_t = \alpha_t x_1 + \sigma_t x_0$, where $\alpha_T = 1$, $\alpha_0 = 0$, $\sigma_T = 0$, $\sigma_0 = 1$. We let $F_t(x_t)$ denote a vector map of force, a trainable neuron network parameterized with $\theta$. We select the $\alpha_t$ and $\sigma_t$ identical with TrigFlow \cite{ls24}, where $\alpha_t = \sin(t)$ and $\sigma_t = \cos(t)$, $T = \frac{\pi}{2}$. Then, force interpolation could be simplified to 
\begin{align*}
    f_t(x_t) =
    & ~ \frac{(\cos(t)x_1 - \sin(t)x_0) \cdot (-\sin(t)x_1 - \cos(t)x_0)}{c^2 - (\cos(t)x_1 - \sin(t)x_0)^2}\\
    & ~ (\cos(t)x_1 - \sin(t)x_0)).
\end{align*}
\end{theorem}

Theorem~\ref{thm:form_trig:informal} highlights how the ForM framework can be directly coupled with trigonometric interpolation paths. By choosing $\alpha_t$ and $\sigma_t$ as sine and cosine, respectively, we obtain a closed-form expression for the relativistic force that governs the sample evolution. This synergy suggests that ForM can not only unify different flow-based or diffusion-based models but also inherit the continuous-time advantages of the TrigFlow parameterization. Consequently, one can design more flexible interpolation strategies while still enjoying the stability benefits of relativistic velocity constraints.

\section{Experiment} \label{sec:exp}

\begin{algorithm}[!ht]
\caption{First-order Training, \cite{lgl22}}
\label{alg:first_order_training}
\begin{algorithmic}[1]
\Procedure{1stTraining}{$\theta, D$}
\State \Comment{Parameter $\theta$ for the model $u_1$}
\State \Comment{Training dataset $D$}
\While{not converged}
\State {\color{blue} /* Sample a trajectory from dataset. */ }
\State {\color{blue} /*$x_t$ denotes the position at time $t$. */ }
\State {\color{blue} /*$\dot x_t$ denotes the velocity at time $t$. */ }
\State $\{x_t\}_{t \in [0, 1]}, \{\dot x_t\}_{t \in [0, 1]} \sim D$
\State {\color{blue} /* Random sample a timestep $t$. */ }
\State $t \sim \mathcal{U}(0, 1)$
\State {\color{blue} /* Update model parameter $\theta$. */ }
\State $\theta \gets \nabla_\theta ( \| u_1(x_t, t) - \dot x_t \|_2^2)$
\EndWhile
\State \Return{$\theta$}
\EndProcedure
\end{algorithmic}
\end{algorithm}

\begin{algorithm}
[!ht]
\caption{First-order Sampling, \cite{lgl22}}
\label{alg:first_order_sampling}
\begin{algorithmic}[1]
\Procedure{1stdSampling}{$\theta, M$}
\State \Comment{Parameter $\theta$ for the model $u_1$}
\State \Comment{The number of sampling steps $M$}
\State $x \sim \N (0, I)$
\State $d \gets 1 / M$
\State $t \gets 0$
\For{$n \in [0, \dots, M - 1]$}

\State $x \gets x + d \cdot u_1 (x, t)$
\State $t \gets t + d$
\EndFor
\State \textbf{return} $x$
    \EndProcedure
\end{algorithmic}
\end{algorithm}

\begin{algorithm}[!ht]
\caption{First and Second-order Training}
\label{alg:first_second_order_training}
\begin{algorithmic}[1]
\Procedure{2ndTraining}{$\theta, D$}
\State \Comment{Parameter $\theta$ for the models $u_1$ and $u_2$} 
\State \Comment{Training dataset $D$}
\While{not converged}
\State {\color{blue} /* Random sample a trajectory from dataset. */ }
\State {\color{blue} /*$x_t$ denotes the position at time $t$. */ }
\State {\color{blue} /*$\dot x_t$ denotes the velocity at time $t$. */ }
\State {\color{blue} /*$\ddot x_t$ denotes the acceleration at time $t$. */ }
\State $\{x_t\}_{t \in [0, 1]}, \{\dot x_t\}_{t \in [0, 1]}, \{\ddot x_t\}_{t \in [0, 1]} \sim D$
\State {\color{blue} /* Random sample a timestep $t$. */ }
\State $t \sim \mathcal{U}(0, 1)$
\State {\color{blue} /* Update model parameter $\theta$. */ }
\State $\theta \gets \nabla_\theta ( \| u_1(x_t, t) - \dot x_t \|_2^2$
\Statex \hspace{4.2em} $ + \| u_2 (u_1 (x_t, t), x_t, t) - \ddot x_t \|_2^2)$
\EndWhile
\State \Return{$\theta$}
\EndProcedure
\end{algorithmic}
\end{algorithm}

\begin{algorithm}
[!ht]
\caption{First and Second-order Sampling}
\label{alg:first_second_order_sampling}
\begin{algorithmic}[1]
\Procedure{2ndSampling}{$\theta, M$}
\State \Comment{Parameter $\theta$ for the models $u_1$ and $u_2$}
\State \Comment{The number of sampling steps $M$}
\State $x \sim \N (0, I)$
\State $d \gets 1 / M$
\State $t \gets 0$
\For{$n \in [0, \dots, M - 1]$}

\State $x \gets x + d \cdot u_1 (x, t) + \frac{d^2}{2} \cdot u_2 (u_1 (x, t), x, t)$
\State $t \gets t + d$
\EndFor
\State \textbf{return} $x$
    \EndProcedure
\end{algorithmic}
\end{algorithm}

\begin{algorithm}[!ht]
\caption{ForM Training, (Ours)}
\label{alg:form_training}
\begin{algorithmic}[1]
\Procedure{ForMTraining}{$\theta, D, p, k$}
\State \Comment{Parameter $\theta$ for the ForM model $F_\theta$}
\State \Comment{Training dataset $D$}
\State \Comment{Stepsize and time index distribution $p$}
\State \Comment{Batch size $k$}
\While{not converged}
\State {\color{blue} /* Random sample a trajectory from dataset. */ }
\State {\color{blue} /*$x_t$ denotes the position at time $t$. */ }
\State {\color{blue} /*$f_t$ denotes the Lorentz force at time $t$. */ }
\State $\{x_t\}_{t \in [0, 1]}, \{ f_t \}_{t \in [0, 1]} \sim D$
\State {\color{blue} /* Random sample a timestep $t$. */ }
\State $t \sim \mathcal{U}(0, 1)$
\State {\color{blue} /* Update model parameter $\theta$. */ }
\State $\theta \gets \nabla_\theta ( \| F_\theta(t) - f_t \|_2^2)$
\EndWhile
\State \Return{$\theta$}
\EndProcedure
\end{algorithmic}
\end{algorithm}

\begin{algorithm}
[!ht]
\caption{ForM Sampling, (Ours)}
\label{alg:form_sampling}
\begin{algorithmic}[1]
\Procedure{ForMSampling}{$\theta, M$}
\State \Comment{Parameter $\theta$ for the ForM model $F_\theta$}
\State \Comment{The number of sampling steps $M$}
\State $x \sim \N (0, I)$
\State $d \gets 1 / M$
\State $t \gets 0$
\For{$n \in [0, \dots, M - 1]$}
\State {\color{blue}/* 
Get Lorentz force $f_L$ via model $F_\theta$. */} 
\State $f_L \gets F_\theta(t)$
\State {\color{blue}/* 
Calculate acceleration $a$ via Theorem~\ref{thm:ode_form:informal}. */} 
\State $a \gets \phi (v_\mathrm{prev}, f_L)$ 
\State {\color{blue}/* 
Calculate velocity $v$. */} 
\State $v \gets v_\mathrm{prev} + d 
~ a_t$
\State {\color{blue}/* 
Calculate $x$. */} 
\State $x \gets x_\mathrm{prev} + d ~ (v_t + v_{\mathrm{prev}}) / 2$
\State {\color{blue}/* 
Update $t$. */} 
\State $t \gets t + d$
\State {\color{blue}/* Store variables for next iteration. */}
\State $x_{\mathrm{prev}} \gets x$
\State $v_{\mathrm{prev}} \gets v$
\EndFor
\State \textbf{return} $x$
    \EndProcedure
\end{algorithmic}
\end{algorithm}

\begin{algorithm}
[!ht]
\caption{Numerical ODE Solver}
\begin{algorithmic}[1]
\Procedure{NumericalODESolver}{$\theta, M$}
\State \Comment{Parameter $\theta$ for the ForM model $F_t$}
\State \Comment{The number of sampling steps $M$}
\State $x \sim \N (0, I)$
\State $d \gets 1 / M$
\State $t \gets 0$
\For{$n \in [0, \dots, M - 1]$}
\State {\color{blue}/* Solve the ODE in Theorem~\ref{thm:ode_form:informal}. */} 
\State $x \gets \mathrm{ODE Solver}(x, \frac{\d x}{\d t}, F(t), d)$ 
\State $t \gets t + d$
\EndFor
\State \textbf{return} $x$
    \EndProcedure
\end{algorithmic}
\end{algorithm}

In this section, we conduct a systematic evaluation of Force Matching (ForM) effectiveness through extensive experimentation, emphasizing its significant role in advancing distribution generation. Section~\ref{sec:exp:experiment_setup} describes the experimental framework, encompassing the Onedot, Halfmoons, and Spiral datasets, force dynamics, and a comparative analysis of various trajectory evolution methods. Section~\ref{sec:exp:result_anapysis} examines qualitative outcomes, demonstrating how second-order terms refine trajectory smoothness and highlighting ForM’s superiority in transport dynamics modeling. 

\subsection{Experiment setup} \label{sec:exp:experiment_setup}
For an in-depth trajectory evolution analysis, we assess three approaches: the standard flow matching~\cite{lcb+22} method utilizing only the first-order term, an improved approach integrating both first-order and second-order terms, and our proposed ForM model on three kinds of complex and challenging datasets.

\paragraph{Datasets.} (1) Onedot dataset: As illustrated in Figure~\ref{fig:onedot_dataset}, the Onedot dataset comprises 200 points sampled from a Gaussian distribution with variance $0.3$ to establish the central source distribution. The target distribution is generated via a Lorentz field, where each point's initial velocity matches its initial position vector. The parallel force is defined as $\gamma^3 m_0 a_x$, while the perpendicular force follows $\gamma m_0 a_x$, where $\gamma = (1 - v^2 / c^2)^{-\frac{1}{2}}$, and the speed of light is set to \(c = 3 \times 10^8\,\text{m/s}\). In the Onedot dataset, both forces are set to $1.5 \times 10^8$, with a duration of $1s$. 

(2) Halfmoons dataset:
Then, as illustrated in Figure~\ref{fig:halfmoons_dataset}, we sampled 1000 dots from a central source distribution, and the target distribution is generated by a Lorentz field. For the dots above the $x$-axis, we set the direction of their velocity parallel to the $x$-axis to the right, and the dots below the $x$-axis are parallel to the $x$-axis to the left. And the parallel force denoted as $1 \times 10^7 \cdot \sin( t )$, the perpendicular force follows $7 \times 10^8 \cdot \sin( 8 t )$. 

(3) Spiral dataset: 
With reference to the Spiral dataset as shown in Figure~\ref{fig:spiral_dataset}, 1000 data points were extracted from a central source distribution, and the target distribution was similarly synthesized via a Lorentz field. In this instance, the source distribution was segmented into two components—a circular core and an annular ring—by equally partitioning the original radius. The initial velocity assigned to each point was determined by both its designated segment and the angular coordinate of its starting position. As a result, points residing in the ring attain comparatively greater initial speeds, and those with larger angular coordinates likewise manifest enhanced velocities. Finally, the parallel force is specified by $1 \times 10^7 \cdot \sin( t )$, while the perpendicular force is given by $7 \times 10^8 \cdot \sin( t )$.

\paragraph{Baselines.} We compare our ForM model with two traditional yet powerful baselines. The first baseline is standard flow matching, which models the transfer trajectory between distributions using the first-order velocity of the trajectory. We denote this method as O1 (Algorithm~\ref{alg:first_order_training} and \ref{alg:first_order_sampling}). The second baseline is an improved version of the first-order approach, incorporating both first-order velocity and second-order acceleration to model the trajectory, which we denote as O1+O2 (Algorithm~\ref{alg:first_second_order_training} and \ref{alg:first_second_order_sampling}). Both baselines rely on modeling the higher-order derivatives of the trajectory. In contrast, our ForM model (Algorithm~\ref{alg:form_training} and \ref{alg:form_sampling}) introduces a novel approach by modeling the Lorentz force acting on the trajectory within a field, offering a fundamentally different perspective on learning the target distribution. This represents a new modeling paradigm that extends beyond conventional flow-based methods.

\begin{figure}[!ht]
\centering
\includegraphics[width=0.235\textwidth]{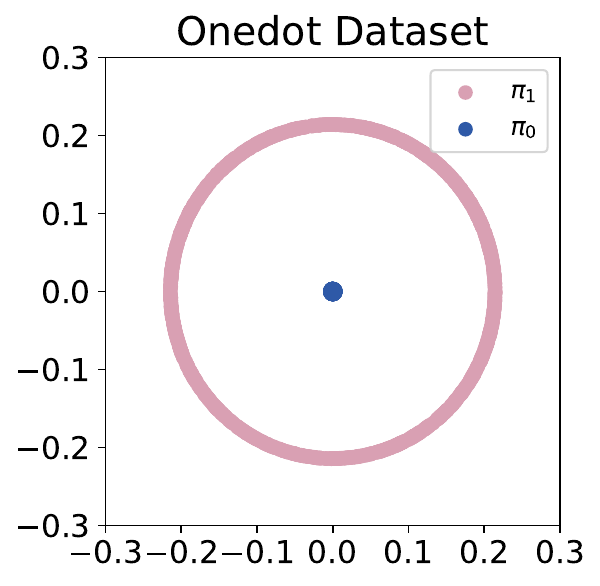} 
\includegraphics[width=0.235\textwidth]{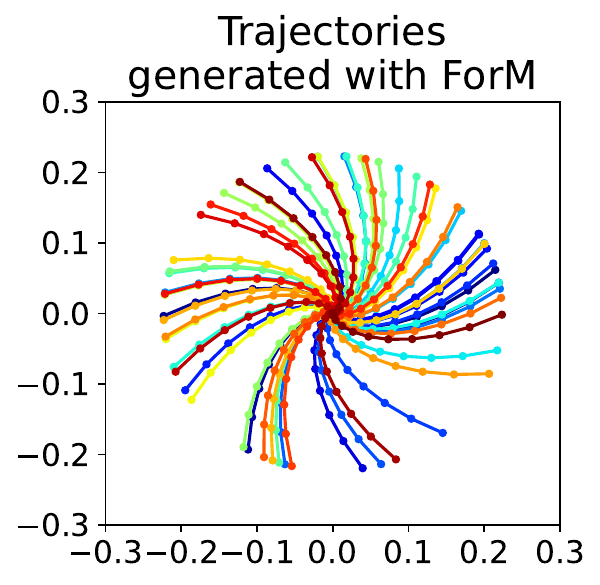}
\caption{Left: Onedot Dataset. The objective is to train the ForM model to learn a transport trajectory from distribution $\pi_0$ ({\textbf{blue}}) to distribution $\pi_1$ ({\textbf{pink}}). Right: The transportation trajectory generated by the ForM model.}
\label{fig:onedot_dataset}
\ifdefined\isarxiv
\else
\Description{}
\fi
\end{figure}

\begin{figure}[!ht]
\centering
\includegraphics[width=0.235\textwidth]{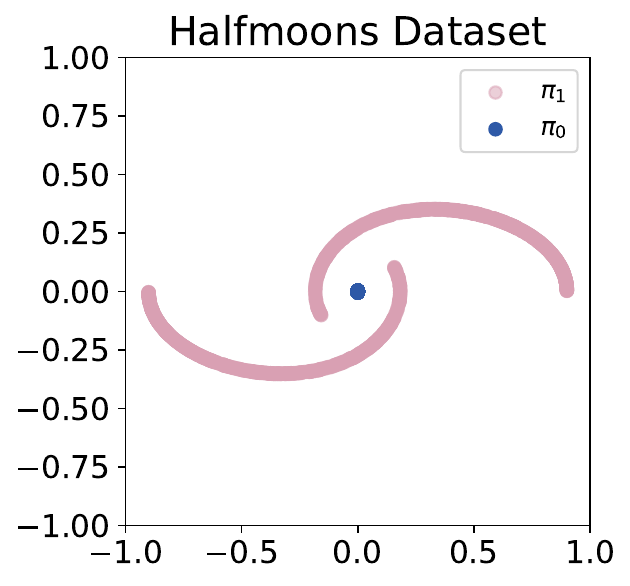}
\includegraphics[width=0.235\textwidth]{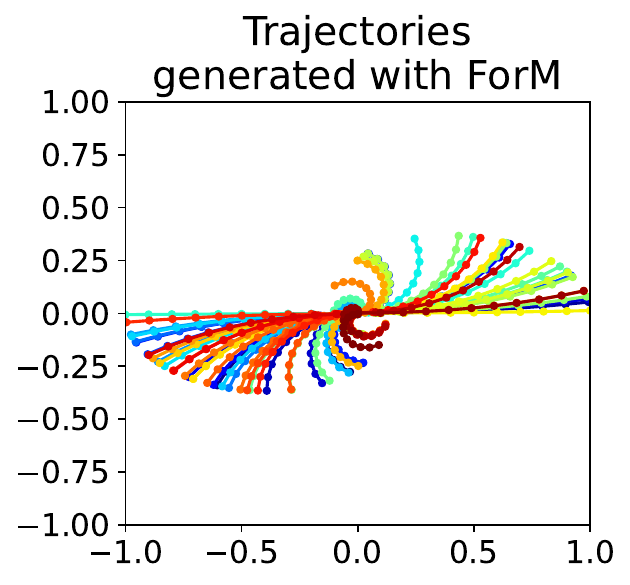} 
\caption{Left: Halfmoons Dataset. The objective is to train ForM to learn a transport trajectory from distribution $\pi_0$ ({\textbf{blue}}) to distribution $\pi_1$ ({\textbf{pink}}). Right: The transportation trajectory generated by the ForM model.}
\label{fig:halfmoons_dataset}
\ifdefined\isarxiv
\else
\Description{}
\fi
\end{figure}

\begin{figure}[!ht]
\centering
\includegraphics[width=0.235\textwidth]{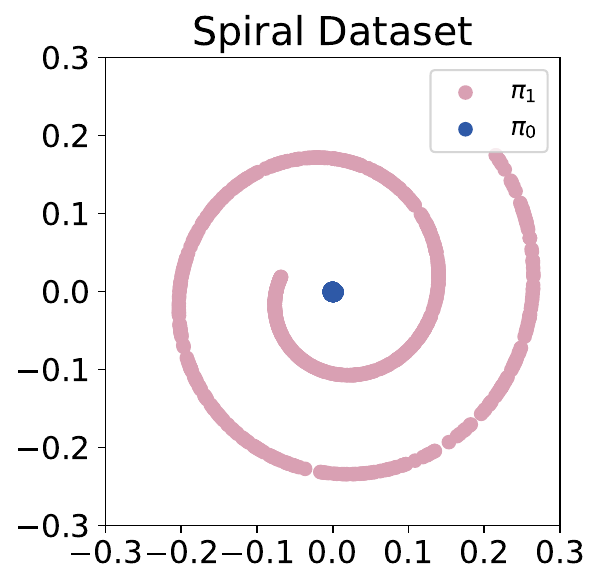}
\includegraphics[width=0.235\textwidth]{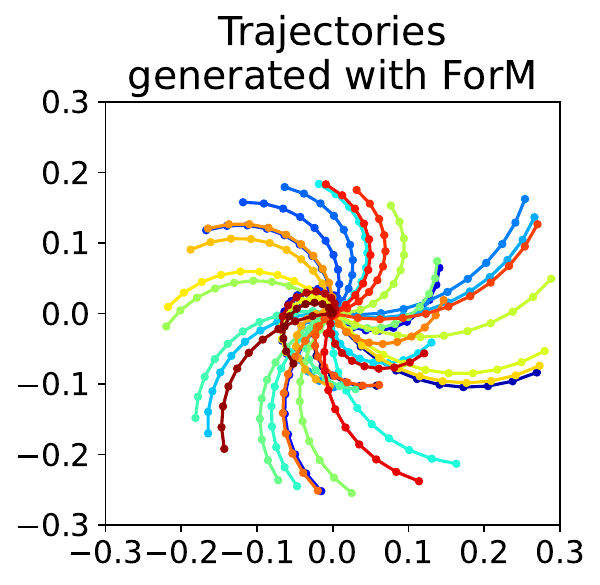}
\caption{Left: Spiral Dataset. The objective is to train the ForM model to learn a transport trajectory from distribution $\pi_0$ ({\textbf{blue}}) to distribution $\pi_1$ ({\textbf{pink}}). Right: The transportation trajectory generated by the ForM model.}
\label{fig:spiral_dataset}
\ifdefined\isarxiv
\else
\Description{}
\fi
\end{figure}

\begin{figure*}[!ht]
\centering
\subfloat[O1, \cite{lgl22}]{\includegraphics[width=0.3\textwidth]{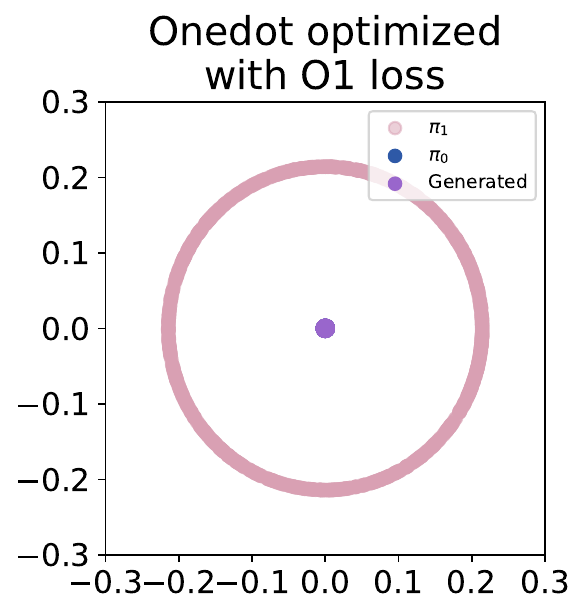}}
\subfloat[O1+O2]{\includegraphics[width=0.3\textwidth]{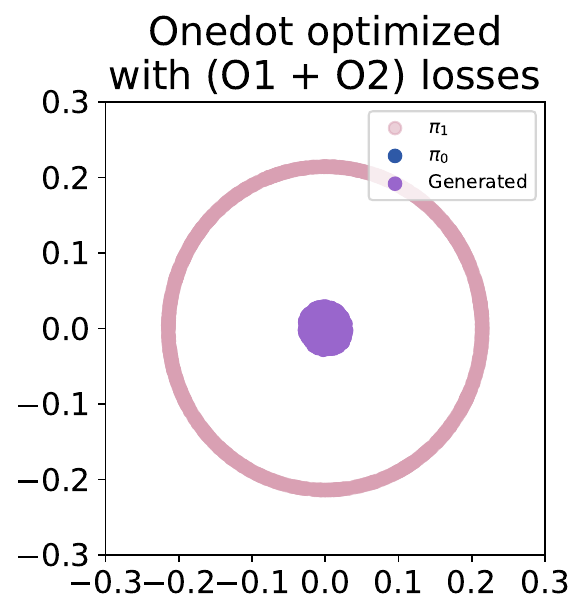}}
\subfloat[ForM, (Ours)]{\includegraphics[width=0.3\textwidth]{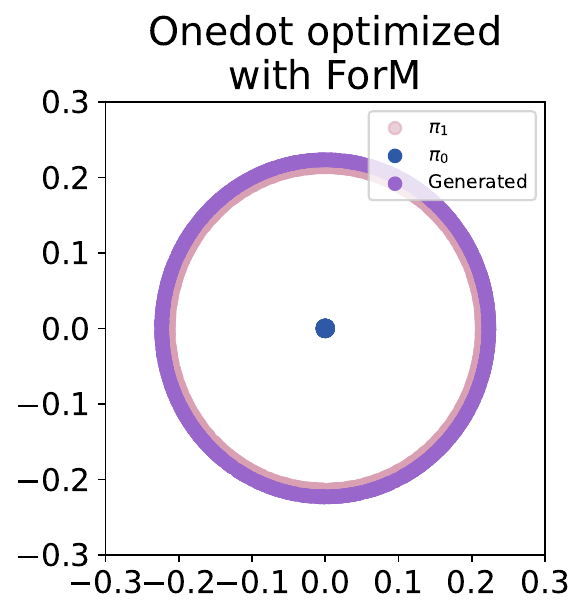}}
\\
\subfloat[O1, \cite{lgl22}]{\includegraphics[width=0.3\textwidth]{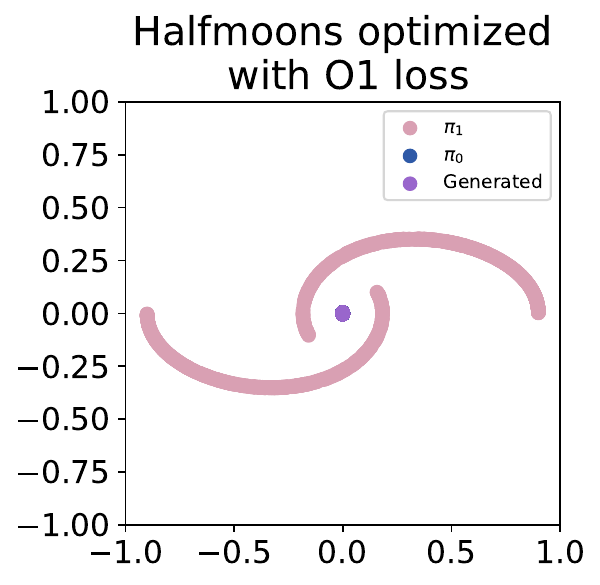}}
\subfloat[O1+O2]{\includegraphics[width=0.3\textwidth]{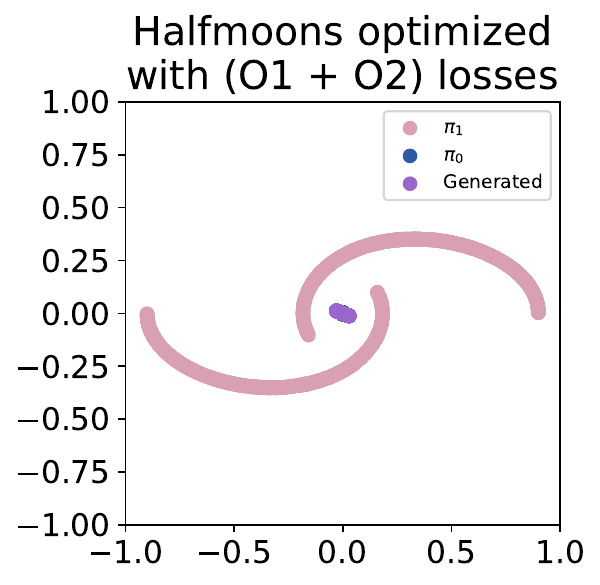}}
\subfloat[ForM, (Ours)]{\includegraphics[width=0.3\textwidth]{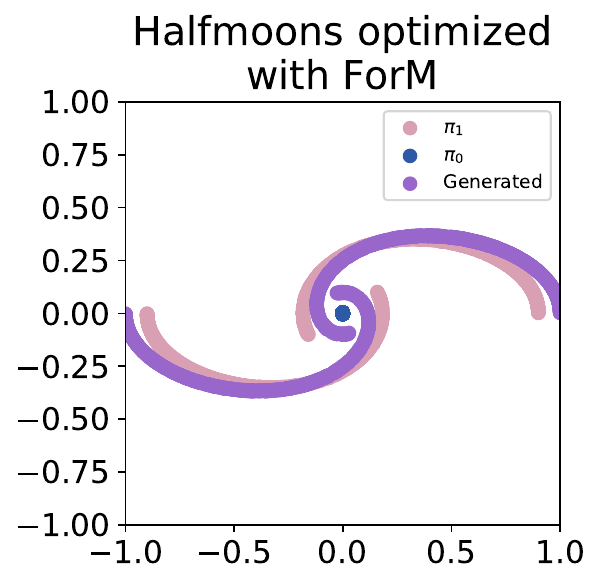}} \\
\subfloat[O1, \cite{lgl22}]{\includegraphics[width=0.3\textwidth]{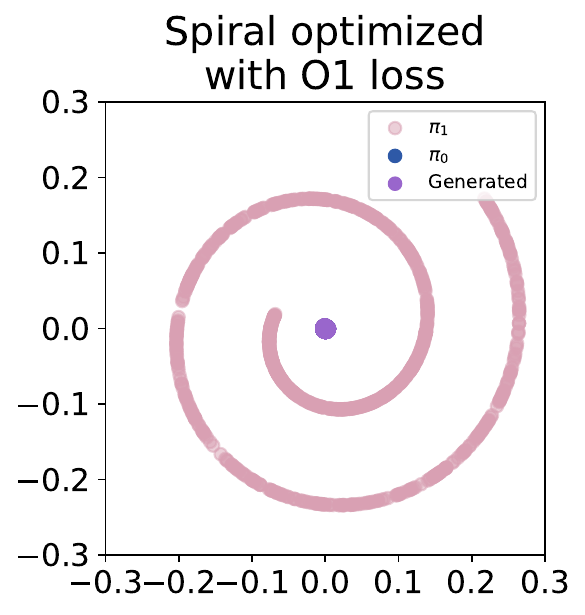}}
\subfloat[O1+O2]{\includegraphics[width=0.3\textwidth]{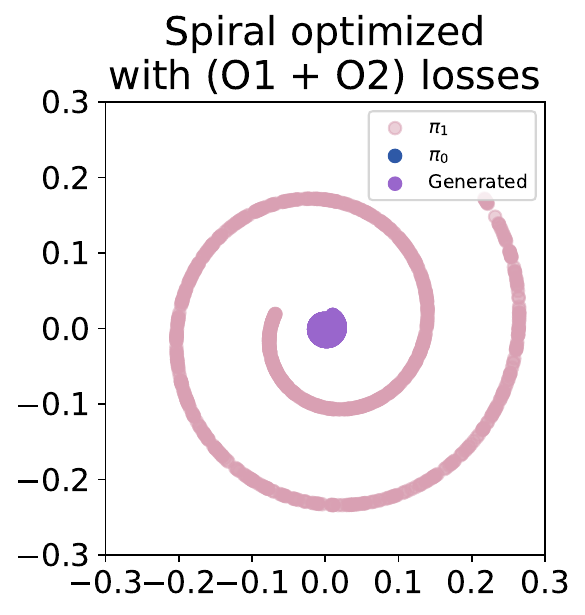}}
\subfloat[ForM, (Ours)]{\includegraphics[width=0.3\textwidth]{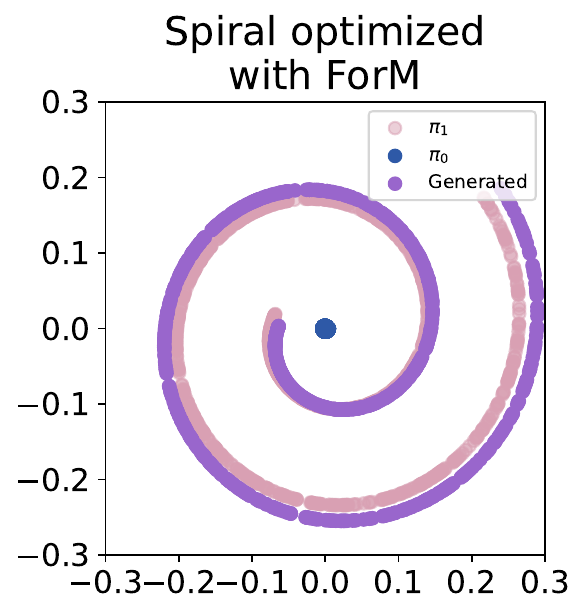}}
\caption{
\textbf{Left:} Flow matching~\cite{lcb+22} using only the first-order term.
\textbf{Middle:} An improved method that incorporates both first- and second-order terms.
\textbf{Right:} Our proposed ForM model applied to the Onedot, Halfmoons, and Spiral datasets.  
Note that the first-order method (O1) fails to capture the target distribution, and although the second-order method (O2) exhibits slight improvement, it still does not adequately model the target distribution. In contrast, by leveraging the Lorentz force to guide the trajectory evolution, the ForM model significantly enhances the accuracy of the target distribution, as further evidenced by the quantitative results in Table~\ref{tab:euclidean_distance_complex_datasets_new}.
}
\label{fig:onedot}
\ifdefined\isarxiv
\else
\Description{}
\fi
\end{figure*}

\subsection{Results analysis} \label{sec:exp:result_anapysis}

This section presents a comparative analysis of trajectory evolution on the \textit{Onedot} and \textit{Halfmoons} datasets. Figure~\ref{fig:onedot} illustrates the trajectories generated by three approaches: (i) the original flow matching method that utilizes only the first-order term (denote as O1), (ii) an enhanced variant that incorporates both first- and second-order terms (denote as O1+O2), and (iii) our proposed Force Matching model (denote as ForM) which explicitly integrates the influence of a Lorentz force field.

Although the inclusion of the second-order term introduces additional dynamic information, the enhanced variant still fails to capture the complex transportation trajectories required by these datasets. In contrast, by modeling the potential Lorentz force, the ForM model markedly improves performance. This integration acts as a guiding mechanism, effectively steering the model toward a more accurate approximation of the target distribution. In other words, the physics-inspired addition of the Lorentz force facilitates the learning process, enabling the model to navigate the intricacies of the complex transport dynamics more efficiently.

This qualitative improvement is further substantiated by the quantitative results reported in Table~\ref{tab:euclidean_distance_complex_datasets_new}, where the ForM model achieves the lowest Euclidean distance loss among the three methods. These findings confirm that incorporating domain-specific forces, such as the Lorentz force, can significantly enhance the modeling of complex transport phenomena.

\subsection{Euclidean distance loss}\label{sec:exp:euclidean_distance_loss}
This section quantitatively assesses different approaches using Euclidean distance loss, which quantifies the deviation between transported and target distributions. Table~\ref{tab:euclidean_distance_complex_datasets_new} presents the loss values across three methods on the Onedot dataset and Halfmoons dataset, the unit we use is 0.1 light seconds. Lower values indicate higher precision in distribution transfer. The results demonstrate that incorporating the second-order term (O1 + O2) enhances performance beyond the first-order term (O1) alone. 
Notably, ForM surpasses both baselines, achieving the lowest Euclidean distance loss, reinforcing its effectiveness in modeling transport trajectories. 

\begin{table}[!ht]
\begin{center}
\caption{ \textbf{Euclidean distance loss across three methods: Flow matching~\cite{lcb+22} with the first order term ({\textbf{O1}}), an advanced version integrating both first order and second order terms ({\textbf{O1 + O2}}), and our proposed ForM on the Onedot dataset and Halfmoons dataset.} Lower values represent greater accuracy in distribution transfer. The optimal values are displayed in \textbf{bold}, while the second-best values are \underline{underlined}. The qualitative results are provided in Figure~\ref{fig:onedot}. }
\label{tab:euclidean_distance_complex_datasets_new}
\begin{tabular}{|l|c|c|c|}
\hline
\textbf{Loss terms}  & \textbf{Onedot} & \textbf{Halfmoons} & \textbf{Spiral} \\
\hline
O1, \cite{lgl22}             & 2.146 & 5.853 & 1.666 \\
\hline
O1 + O2          & \underline{2.048} & \underline{5.793} & \underline{1.578} \\
\hline
ForM, (Ours)      & \textbf{0.509} & \textbf{0.714} & \textbf{0.124} \\
\hline
\end{tabular}
\end{center}
\end{table}

\section{Empirical Ablation Study} \label{sec:ablation}

In this section, we empirically evaluate the performance of our proposed ForM model through a series of ablation experiments. We begin by illustrating the Spiral dataset and the corresponding transport objective, where samples are moved from the distribution $\pi_0$ ({\textbf{blue}}) to $\pi_1$ ({\textbf{pink}}) (see Figure~\ref{fig:spiral_dataset}). We then compare three approaches: the baseline Flow matching method using solely the first-order term (O1), an enhanced variant that integrates both first-order and second-order loss metrics (O1 + O2), and our ForM model that leverages the Lorentz force to guide trajectory evolution. Qualitative results are shown in Figure~\ref{fig:onedot}, while Table~\ref{tab:euclidean_distance_complex_datasets_new} quantitatively demonstrates the superior performance of the ForM model.

\section{Discussion}

The core motivation behind the Force Matching (ForM) model stems from the need to stabilize generative modeling processes, particularly in flow-based methods where uncontrolled velocity magnitudes can lead to instability during sampling. Traditional approaches, such as Flow Matching (FM), offer different perspectives on modeling data evolution, yet they often lack explicit constraints that regulate sample movement, which can hinder both stability and efficiency. Inspired by special relativistic mechanics, we introduce a principled way to control velocity magnitudes through the Lorentz factor, ensuring that sample velocities remain bounded during the generative process. This formulation draws a parallel between relativistic motion and generative trajectories, where limiting sample speed prevents erratic behaviors and enhances robustness. By enforcing this constraint, ForM provides a novel perspective on generative modeling that aligns with both theoretical principles and practical stability considerations. 

While ForM establishes a strong foundation for stable generative modeling, several exciting directions remain open for further exploration. Extending ForM to high-dimensional image and video generation tasks would be a crucial next step, requiring efficient implementations of relativistic constraints in large-scale neural networks. Additionally, investigating how ForM’s velocity constraint can be integrated into score-based generative models could lead to hybrid approaches that combine the strengths of both paradigms. Another promising direction is the development of adaptive or learnable velocity constraints that dynamically regulate sample movement based on data complexity, potentially enhancing flexibility. More broadly, the incorporation of relativistic principles into generative modeling raises questions about the role of physics-inspired constraints in deep learning.
Finally, the implicit connection between ForM and optimal transport theory suggests that deeper theoretical investigations in this direction could lead to new generative frameworks grounded in optimal transport principles. 
By leveraging insights from physics and generative modeling, ForM paves the way for designing more stable, efficient, and interpretable generative models, inspiring further research into force-based approaches and their applications to complex data synthesis tasks.

\section{Conclusion} \label{sec:conclusion}

In this work, we introduce Force Matching (ForM), an innovative and comprehensive framework for generative modeling. ForM incorporates the principles of relativistic mechanics, higher-order flow matching, and TrigFlow, forming a unique and powerful synergy. The core idea behind ForM is to model the generative process in a way that accounts for both the geometric and dynamic aspects of flow, inspired by relativistic mechanics. We demonstrate that ForM not only preserves the structure of the data but also introduces an additional layer of stability to the generative process. Specifically, we theoretically prove that ForM bounds the velocity of the generative process under a hyperparameter $c$ during the sampling procedure, which leads to improved control and stabilization of the process. This stabilization mechanism mitigates issues such as mode collapse and sampling instability that often plague other generative models. 
Through extensive empirical experiments, we demonstrate that ForM outperforms both Flow Matching and second-order Flow Matching in terms of generative quality and sample diversity. These results highlight the effectiveness of incorporating higher-order dynamics and relativistic principles into the generative process. 
Additionally, we conduct an ablation study to evaluate the individual components of ForM, further demonstrating its superiority over existing methods. 
This analysis confirms the contribution of each aspect of the framework, such as the higher-order flow matching and relativistic dynamics, to its overall performance. 
ForM offers a fresh perspective on the understanding of Flow Matching within the context of relativistic mechanics, presenting a new paradigm in the field of generative modeling. 
By redefining the conceptual and practical foundations of generative processes, ForM sets a new benchmark for the future development of generative models.

\newpage
\ifdefined\isarxiv
\bibliographystyle{alpha}
\bibliography{ref}
\else
\bibliography{ref}
\bibliographystyle{ACM-Reference-Format} 
\fi

\newpage 
\appendix
\ifdefined\isarxiv
\begin{center}
    \textbf{\LARGE Appendix}
\end{center}
\else
\section*{Appendix}
\fi
\paragraph{Roadmap.}
In Section~\ref{sec:miss_proof}, we provide a formal version of theoretical analysis and proofs.

\section{Theoretical Analysis} \label{sec:miss_proof}

In this section, we first provide the formal theorem and proof for the sampling ODE in Section~\ref{sub:app:samp_ode}. Then, we formally proved the speed limit of ForM's sampling ODE in Section~\ref{sub:app:speed_limit}. In Section~\ref{sub:app:form_trig}, we formally prove the derivation of the interpolation path of ForM with TrigFlow. Last, we illustrate the formal proof for relativistic force in Section~\ref{sub:app:force}.

\subsection{Sampling ODE} \label{sub:app:samp_ode}

Here, we restate the Theorem~\ref{thm:ode_form:informal} and state its proof.

\begin{theorem}[Sampling ODE, formal version of Theorem~\ref{thm:ode_form:informal}]\label{thm:ode_form:formal}
    Giving the force at position $x_t$ denoted as $f_t(x_t)$, we could solve for ForM sampling path $x_t$ by the following ODE
    \begin{align*}
    \ddot{x}_t = \frac{1}{m^{\rm lab} \gamma_t}(f_t^{\rm local} - \frac{\langle v_t^{\rm lab}, f_t^{\rm local} \rangle}{c^2} v_t^{\rm lab})
    \end{align*}
    where $x_0 \sim \N(0,I)$, $\dot{x}_0 = 0$.
\end{theorem}

\begin{proof}
Recall $f^{\rm local}$ from Lemma~\ref{lem:equiv_relativistic_force:formal}
\begin{align*}
    f_t^{\rm local} = m^{\rm lab}  (\gamma_t a_t^{\rm lab} + \gamma_t^3 \frac{ \langle v_t^{\rm lab}, a_t^{\rm lab} \rangle}{c^2} v_t^{\rm lab}),
\end{align*}
where $\gamma_t$ is the Lorentz factor defined in Definition~\ref{def:LorentzFactor}.

To solve for $a_t^{\rm lab}$, we could first decompose $a_t^{\rm lab}$ by
\begin{align*}
    a_t^{\rm lab} = a_{t,\parallel}^{\rm lab} + a_{t,\perp}^{\rm lab},
\end{align*}
where $a_{t,\parallel}^{\rm lab}$ denotes the component of $a_t^{\rm lab}$ parallel with $v_t^{\rm lab}$, and $a_{t,\perp}^{\rm lab}$ denotes the component of $a_t^{\rm lab}$ perpendicular with $v_t^{\rm lab}$.

According to the definition of parallel and perpendicular, we have
\begin{align*}
    a_{t,\parallel}^{\rm lab} = & ~ \frac{ \langle v_t^{\rm lab}, a_t^{\rm lab} \rangle}{\|v_t^{\rm lab}\|_2^2} v_t^{\rm lab}, \\
    a_{t,\perp}^{\rm lab} = & ~ a_t^{\rm lab} - a_{t,\parallel}^{\rm lab}.
\end{align*}

Then we have
\begin{align}
    f_t^{\rm local} = & ~ m^{\rm lab}  (\gamma_t a_t^{\rm lab} + \gamma_t^3 \frac{ \langle v_t^{\rm lab}, a_{t}^{\rm lab} \rangle}{c^2} v_t^{\rm lab}) \notag \\
    = & ~ m^{\rm lab}  (\gamma_t (a_{t,\parallel}^{\rm lab} + a_{t,\perp}^{\rm lab}) + \gamma_t^3 \frac{ \langle v_t^{\rm lab}, a_{t,\parallel}^{\rm lab} + a_{t,\perp}^{\rm lab} \rangle}{c^2} v_t^{\rm lab}) \notag \\
    = & ~ m^{\rm lab}  (\gamma_t (a_{t,\parallel}^{\rm lab} + a_{t,\perp}^{\rm lab}) + \gamma_t^3 \frac{ \langle v_t^{\rm lab}, a_{t,\parallel}^{\rm lab} \rangle}{c^2} v_t^{\rm lab}), \label{eq:f_split}
\end{align}
where the first step follows Lemma~\ref{lem:equiv_relativistic_force:formal}, the second step decomposes $a_t^{\rm lab}$, and the last step follows from the simple fact that $\langle v_t^{\rm lab}, a_{t,\perp}^{\rm lab} \rangle = 0$.

Then we decompose the $f_t^{\rm local}$ to $f_{t, \parallel}^{\rm local}$ and $f_{t, \perp}^{\rm local}$, where $f_{t, \parallel}^{\rm local}$ denotes the component of $f_t^{\rm local}$ parallel with $v_t^{\rm lab}$, and $f_{t,\perp}^{\rm local}$ denotes the component of $f_t^{\rm local}$ perpendicular with $v_t^{\rm lab}$.

For the perpendicular component, we have
\begin{align*}
    f_{t, \perp}^{\rm local} = & ~ m^{\rm lab}\gamma_t  a_{t,\perp}^{\rm lab} \\
    a_{t,\perp}^{\rm lab} =  & ~ \frac{f_{t, \perp}^{\rm local}}{m^{\rm lab}\gamma_t},
\end{align*}
where the first step uses the perpendicular part from Eq.~\ref{eq:f_split}, and the second step rewrites the equation to get a closed-form solution for $a_{t,\perp}^{\rm lab}$.

For the parallel component, we have
\begin{align*}
    f_{t, \parallel}^{\rm local} = & ~ m^{\rm lab} (\gamma_t  a_{t,\parallel}^{\rm lab} + \gamma_t^3 a_{t,\parallel}^{\rm lab} \frac{\|v_t^{\rm lab}\|_2^2}{c^2}) \\
    = & ~ m^{\rm lab} a_{t,\parallel}^{\rm lab}(\gamma_t + \gamma_t^3 \frac{\|v_t^{\rm lab}\|_2^2}{c^2}) \\
    a_{t,\parallel}^{\rm lab} = & ~ \frac{f_{t, \parallel}^{\rm local}}{m^{\rm lab}(\gamma_t + \gamma_t^3 \frac{\|v_t^{\rm lab}\|_2^2}{c^2})},
\end{align*}
where the first step uses the parallel part from Eq.~\ref{eq:f_split}, the second step factors out the $a_{t,\parallel}^{\rm lab}$, and the last step rewrites the equation to get a closed-form solution for $a_{t,\parallel}^{\rm lab}$.

Then, we can combine these two components
\begin{align*}
    a_t^{\rm lab} = & ~\frac{f_{t, \perp}^{\rm local}}{m^{\rm lab}\gamma_t} + \frac{f_{t, \parallel}^{\rm local}}{m^{\rm lab}(\gamma_t + \gamma_t^3 \frac{\|v_t^{\rm lab}\|_2^2}{c^2})} \\
    = & ~\frac{f_t^{\rm local} - \frac{\langle v_t^{\rm lab}, f_t^{\rm local} \rangle}{\|v_t^{\rm lab}\|_2^2} v_t^{\rm lab}}{m^{\rm lab}\gamma_t} + \frac{\frac{ \langle v_t^{\rm lab}, f_t^{\rm local} \rangle}{\|v_t^{\rm lab}\|_2^2} v_t^{\rm lab}}{m^{\rm lab}(\gamma_t + \gamma_t^3 \frac{\|v_t^{\rm lab}\|_2^2}{c^2})} \\
    = & ~\frac{1}{m^{\rm lab}\gamma_t} f_t^{\rm local} + \frac{1}{m^{\rm lab}} ( \frac{1}{\gamma_t(1 + \gamma_t^2 \frac{\|v_t^{\rm lab}\|_2^2}{c^2})} - \frac{1}{\gamma_t} ) \frac{ \langle v_t^{\rm lab}, f_t^{\rm local} \rangle}{\|v_t^{\rm lab}\|_2^2} v_t^{\rm lab} \\
    = & ~\frac{1}{m^{\rm lab}\gamma_t} f_t^{\rm local} + \frac{1}{m^{\rm lab}} ( \frac{1}{\gamma_t \frac{c^2}{c^2 - \|v_t^{\rm lab}\|_2^2} } - \frac{1}{\gamma_t} ) \frac{ \langle v_t^{\rm lab}, f_t^{\rm local} \rangle}{\|v_t^{\rm lab}\|_2^2} v_t^{\rm lab} \\
    = & ~\frac{1}{m^{\rm lab}\gamma_t} f_t^{\rm local} + \frac{1}{m^{\rm lab}} ( \frac{c^2 - \|v_t^{\rm lab}\|_2^2}{\gamma_t c^2} - \frac{1}{\gamma_t} ) \frac{ \langle v_t^{\rm lab}, f_t^{\rm local} \rangle}{\|v_t^{\rm lab}\|_2^2} v_t^{\rm lab} \\
    = & ~\frac{1}{m^{\rm lab}\gamma_t} f_t^{\rm local} - \frac{1}{m^{\rm lab} \gamma_t} \frac{\|v_t^{\rm lab}\|_2^2}{c^2} \frac{ \langle v_t^{\rm lab}, f_t^{\rm local} \rangle}{\|v_t^{\rm lab}\|_2^2} v_t^{\rm lab} \\
    = & ~ \frac{1}{m^{\rm lab} \gamma_t}(f_t^{\rm local} - \frac{\langle v_t^{\rm lab}, f_t^{\rm local} \rangle}{c^2} v_t^{\rm lab}),
\end{align*}
where the first step combines the two terms, the second step decomposes the components, the third step factors out the $\gamma_t$ in denominator, the forth uses the fact that $\gamma_t^2 = \frac{1}{1 - \|v_t^{\rm lab}\|_2^2/c^2}$, the fifth step moves the denominator into numerator, the sixth step merges two terms, and the last step factors out the $\frac{1}{m^{\rm lab} \gamma_t}$.
\end{proof}

\subsection{Speed Limit} \label{sub:app:speed_limit}

In this subsection, we first calculate the derivative of the squared norm of velocity, then restate the Theorem~\ref{thm:vel:informal} and provide its proof.

\begin{lemma}[Derivative of the squared norm of velocity]\label{lem:velocity_derivative}
    Let $X(t) := \frac{1}{2}\|v_t^{\rm lab}\|_2^2$. Then we have
    \begin{align*}
        \frac{\d X(t)}{\d t} = \frac{ \langle f_t^{\rm local}, v_t^{\rm lab} \rangle}{m^{\rm lab} \gamma_t}(1 - \frac{\|v_t^{\rm lab}\|_2^2}{c^2}). 
    \end{align*}
\end{lemma}
\begin{proof}
Recall the sampling ODE
\begin{align}\label{eq:tmp}
    a_t^{\rm lab} = \frac{1}{m^{\rm lab} \gamma_t}(f_t^{\rm local} - \frac{\langle v_t^{\rm lab}, f_t^{\rm local} \rangle}{c^2} v_t^{\rm lab}),
\end{align}
where $\gamma_t$ is the Lorentz factor at lab time $t$ defined in Definition~\ref{def:LorentzFactor}.

We can show that
\begin{align*}
    \frac{\d X(t)}{\d t} = &~ \frac{\d}{\d t} \frac{1}{2}\|v_t^{\rm lab} \|_2^2 \\
    = &~ \langle v_t^{\rm lab}, \frac{\d}{\d t} v_t^{\rm lab} \rangle \\
    = &~ \langle v_t^{\rm lab}, a_t^{\rm lab} \rangle \\
    = &~ \langle v_t^{\rm lab}, \frac{1}{m^{\rm lab} \gamma_t}(f_t^{\rm local} - \frac{\langle v_t^{\rm lab}, f_t^{\rm local} \rangle}{c^2} v_t^{\rm lab}) \rangle \\
    = &~ \frac{1}{m^{\rm lab} \gamma_t} \langle f_t^{\rm local}, v_t^{\rm lab} \rangle - \frac{\langle v_t^{\rm lab}, f_t^{\rm local} \rangle}{m^{\rm lab} \gamma_t c^2} \| v_t^{\rm lab} \|_2^2 \\
    = &~ \frac{ \langle f_t^{\rm local}, v_t^{\rm lab} \rangle}{m^{\rm lab} \gamma_t}(1 - \frac{\|v_t^{\rm lab}\|_2^2}{c^2}) 
\end{align*}
where the first step follows from the definition of $X(t)$, the second step follows from the chain rule, the third step follows from the basic fact, the fourth step follows from the Eq.~\eqref{eq:tmp}, the fifth and last step follows from basic algebra.
\end{proof}

Here, we restate the Theorem~\ref{thm:vel:informal} and state its proof.

\begin{theorem}[Speed Limit, formal version of Theorem~\ref{thm:vel:informal}] \label{thm:vel:formal}
For a ForM model with sampling path $x : [0,T) \to \R^n$, the velocity satisfies
\begin{align*}
\| \dot{x}_t \|_2 < c \quad \text{for all } t \in [0,T).
\end{align*}
\end{theorem}
\begin{proof}
Let $X(t) := \tfrac12 \|v_t^{\rm lab}\|_2^2$. By Lemma~\ref{lem:velocity_derivative}, we have
    \begin{align*}
        \frac{\d X(t)}{\d t} 
        \;=\; \frac{1}{m^{\rm lab} \gamma_t}
        \,\langle f_t^{\rm local}, v_t^{\rm lab} \rangle
        \,\Bigl(1 - \tfrac{\|v_t^{\rm lab}\|_2^2}{c^2}\Bigr).
    \end{align*}
    Observe that the factor $\bigl(1 - \|v_t^{\rm lab}\|_2^2 / c^2 \bigr)$ becomes negative if ever $\|v_t^{\rm lab}\|_2 > c$, and it is zero when $\|v_t^{\rm lab}\|_2 = c$. Thus, if the velocity norm were to exceed $c$ at some time, the derivative of $X(t)$ at that moment would be negative, forcing $X(t)$ (i.e., $\|v_t^{\rm lab}\|_2^2$) to decrease rather than increase. In particular, once $\|v_t^{\rm lab}\|_2^2$ reaches $c^2$, it cannot increase further.

    Consequently, for all $t \in [0,T)$ we must have $\|v_t^{\rm lab}\|_2 < c$, which proves the speed limit. Equivalently, since $\dot{x}_t = v_t^{\rm lab}$ in our notation, we conclude
    \begin{align*}
        \|\dot{x}_t\|_2 < c,
        \quad \forall t \in [0,T).
    \end{align*}
    Thus, we complete the proof.
    \end{proof}

\subsection{ForM with TrigFlow} \label{sub:app:form_trig}

We restate the Theorem~\ref{thm:form_trig:informal} and provide its proof.

\begin{theorem}[ForM with TrigFlow, formal version of Theorem~\ref{thm:form_trig:informal}] \label{thm:form_trig:formal}
We let $m = 1$ for simplicity in ForM. Giving a the interpolation $x_t = \alpha_t x_1 + \sigma_t x_0$, where $\alpha_T = 1$, $\alpha_0 = 0$, $\sigma_T = 0$, $\sigma_0 = 1$. We let $F_t(x_t)$ denote a vector map of force, a trainable neuron network parameterized with $\theta$. We select the $\alpha_t$ and $\sigma_t$ identical with TrigFlow \cite{ls24}, where $\alpha_t = \sin(t)$ and $\sigma_t = \cos(t)$, $T = \frac{\pi}{2}$. Then, force interpolation could be simplified to 
\begin{align*}
    f_t(x_t) =
    & ~ \frac{(\cos(t)x_1 - \sin(t)x_0) \cdot (-\sin(t)x_1 - \cos(t)x_0)}{c^2 - (\cos(t)x_1 - \sin(t)x_0)^2}\\
    & ~ (\cos(t)x_1 - \sin(t)x_0)).
\end{align*}
\end{theorem}
\begin{proof}
    We can show that
    \begin{align*}
        f_t(x_t) = & ~ \gamma_t \ddot{x}_t + \gamma_t^3 \frac{\langle \dot{x}_t, \ddot{x}_t \rangle}{c^2}\dot{x}_t \\
        = & ~ (1 - \frac{(\cos(t)x_1 - \sin(t)x_0)^2}{c^2})^{-\frac{1}{2}} (-\sin(t)x_1 - \cos(t)x_0) + \\
        & ~ \frac{(\cos(t)x_1 - \sin(t)x_0) \cdot (-\sin(t)x_1 - \cos(t)x_0)}{c^2 - (\cos(t)x_1 - \sin(t)x_0)^2}\\
        & ~ (\cos(t)x_1 - \sin(t)x_0)),
    \end{align*}
    where we use Definition~\ref{def:RelativisticForce} and substitute acceleration $a_t^{\rm lab}$ with $\ddot{x}_t$ and $v_t^{\rm lab}$ with $\dot{x}_t$ in the first step, and substitute $x_t = \sin(t) x_1 + \cos(t) x_0$ and calculate it's derivative in the second step.
\end{proof}

\subsection{Relativistic Force Property} \label{sub:app:force}

In this subsection, we restate Lemma~\ref{lem:equiv_relativistic_force:informal}, and show its proof.

\begin{lemma}[Equivalent Form of Relativistic Force, formal version of Lemma~\ref{lem:equiv_relativistic_force:informal}]\label{lem:equiv_relativistic_force:formal}
Let $p^{\rm lab}$ be the momentum defined in Eq.~\eqref{eq:p}, $\gamma_t$ be the Lorentz factor at lab time $t$ defined in Definition~\ref{def:LorentzFactor}, $\tau$ denotes the proper time, $v_t^{\rm lab} = \dot{x}_t$ denotes the velocity, 
$a_t^{\rm lab} = \ddot{x}_t$ denotes the acceleration.
The relativistic force, defined as the time derivative of the momentum in the lab frame, can be written as
\begin{align*}
f^{\rm local} =  m^{\rm lab}  (\gamma_t a_t^{\rm lab} + \gamma_t^3 \frac{ \langle v_t^{\rm lab}, a_t^{\rm lab} \rangle}{c^2} v_t^{\rm lab}).
\end{align*}

\end{lemma}
\begin{proof}
We can show that
\begin{align*}
f^{\rm local}  = & ~ \frac{\d p^{\rm lab}}{\d \tau}  \\
= & ~ \frac{\d m^{\rm lab} v_t^{\rm lab} \gamma_t}{\d t} \\
= & ~ m^{\rm lab} \frac{\d v_t^{\rm lab} \gamma_t}{\d t} \\
= & ~ m^{\rm lab}  (\gamma_t a_t^{\rm lab} + \gamma_t^3 \frac{ \langle v_t^{\rm lab}, a_t^{\rm lab} \rangle}{c^2} v_t^{\rm lab}).
\end{align*}
where the first step follows from Eq.~\eqref{eq:f_local}, the second step follows Eq.~\eqref{eq:p} and Definition~\ref{def:ProperTime}, the third step is true because $m^{\rm lab}$ is a constant, and the last step takes the derivative.
\end{proof}




\end{document}